\documentclass[11pt]{article}
\usepackage[utf8]{inputenc}
\usepackage[margin=1in]{geometry}

\usepackage{amsmath,amsthm,amsfonts,amssymb,mathdots,array,mathrsfs,bm,bbm,stmaryrd,graphicx,subfigure,xcolor}
\usepackage{microtype}
\usepackage{graphicx}
\usepackage{subfigure}
\usepackage{booktabs}
\usepackage{natbib}
\setcitestyle{open={(},close={)}}
\usepackage[colorlinks=true,citecolor=blue]{hyperref}

\usepackage{multicol}

\usepackage[ruled,vlined]{algorithm2e}

\usepackage{amsmath}
\usepackage{amsfonts}
\usepackage{amssymb}
\usepackage{mathbbol}
\usepackage{graphicx}
\usepackage{fullpage}
\usepackage{url}
\usepackage{csquotes}
\usepackage[american]{babel}
\usepackage{comment}
\usepackage{multirow}
\usepackage{url}
\usepackage{times,subfigure}

\newtheorem{theorem}{Theorem}
\newtheorem{lemma}[theorem]{Lemma}

\usepackage{xcolor}

\begin{document}

\title{\huge Planning through Workspace Constraint Satisfaction and Optimization}

\author{\vspace{0.5in}\\\textbf{Weifu Wang} \  and \  \textbf{Ping Li} \\\\
Cognitive Computing Lab\\
Baidu Research\\
10900 NE 8th St. Bellevue, WA 98004, USA\\\\
  \texttt{\{harrison.wfw,  pingli98\}@gmail.com}
}

\date{\vspace{0.5in}}
\maketitle

\begin{abstract}\vspace{0.3in}

\noindent In this work, we present a workspace-based planning framework, which though using redundant workspace key-points to represent robot states, can take advantage of the interpretable geometric information to derive good quality collision-free paths for even complex robots. Using workspace geometries, we first find collision-free piece-wise linear paths for each key point so that at the endpoints of each segment, the distance constraints are satisfied among the key points. Using these piece-wise linear paths as initial conditions, we can perform optimization steps to quickly find paths that satisfy various constraints and piece together all segments to obtain a valid path. We show that these adjusted paths are unlikely to create a collision, and the proposed approach is fast and can produce good quality results.
\end{abstract}

\newpage

\section{Introduction}

We present a workspace-based planning framework, taking advantage of the describable geometries and the simplicity of point-based planning. The proposed method also integrates the constraint Jacobian method used in many planning processes, but instead of taking the gradient with respect to time, we use the gradient to compute the changes the path needs to make to better satisfy the constraints in the optimized trajectories. The proposed method is simple to understand and implement, and the result is comparable to the state-of-the-art sampling-based planning algorithms quality-wise and can help identify the geometrical and topological differences among paths.

Many modern planning algorithms work in the configuration space. The introduction of the configuration space revolutionized the understanding of robot motion and kinematics and helped advance the planning and manipulation significantly. At the same time, in the configuration space, the trajectories can be difficult to visualize, the inverse kinematics changes are hard to derive based on the violations and other constraints, and the configuration space obstacles are hard to describe. Many researchers have attempted to map the workspace geometry to the configuration space geometry, but many unresolved challenges exist.

There has also been an excellent geometrical planning algorithm, including the famous Piano Mover's algorithm series~\citep{schwartz1983pianoI, schwartz1983piano, schwartz1983pianoIII}. These algorithms are mathematically correct and complete but challenging to implement. Moreover, the algorithms can become very complex when concave shapes are involved. The main challenge is that the interactions between different geometries can be complex. In the proposed work, however, we intend to use simple geometries like points and spheres to simplify the geometrical interactions and use geometrical constraints to {\em coordinate} the trajectories among different {\em key-points} on the robot.

The proposed framework can be described in simple terms as the following steps. We conduct workspace decomposition through triangulation and categorize the free regions of the workspace. We extract adjacency relations among the free regions and find paths on the adjacency graph on a high level. We find paths for points starting from various initial positions using straight-line connections. We ensure these points correspond to different key points on the robot and do not collide with the obstacles on these straight-line segments. We then set the distance constraints among the key points on the robot and {\em coordinate} the paths using the initial straight-line segments to satisfy the distance constraints using constraint Jacobian methods. If there are preferred configurations for the robot, they can bias the straight-line paths within the workspace. The paths are then checked for collisions and further sub-divided if collisions were detected. The {\em key-points} include but are not limited to the endpoints for each rigid link.

Piece-wise, the paths can be close to optimal through the {\em coordination} (constraint satisfaction and optimization) process. The line-segment paths are optimal in a 2D workspace using a visibility graph on a global level. In a 3D workspace, such guarantees do not hold, but there are visibility-like algorithms to find good quality paths efficiently~\citep{durand19993d, durand20023d}. We conducted experiments and compared them to sampling-based motion planning algorithms. Our approach is fast and can find good quality paths without expensive optimization or consuming much memory. In particular, we show that our algorithm works well against narrow passages. The proposed approach can be perceived as a divide and conquer approach in the workspace, where the intermediate states subdivide the overall problem into locally solvable geometric problems.

Within the proposed framework, we analyze the nonholonomic constraints. Also, we demonstrate that the proposed framework is effective when narrow passages are present. However, some corner cases may require additional processing within the proposed framework.

\section{Related Work}

Planning is one of the most fundamental problems in robotics. While many efforts have been spent studying ways to generate good or valid paths, many challenges still exist in robotics planning. One challenge is that the robots usually have a high degree of freedom (DOF), leading to high-dimensional configuration spaces. Another challenge is that the complex robot controls can hardly be directly integrated into planning.

There have been studies on planning in the workspace, most notably the Piano Mover's problem series~\citep{schwartz1983pianoI, schwartz1983piano, schwartz1983pianoIII}. The authors used semi-algebraic sets to describe the robots and the obstacles, and the papers showed a provable correct and complete algorithm for simple geometries. However, the approach is also very complex, and the algorithms can be hard to implement. At the same time, as the robots we would like to plan become more and more complex, the geometric planning methods start to give way to sampling-based motion planning algorithms.

In the last 1980s, the sampling-based planning algorithms like Rapid-exploring Random Tree (RRT)~\citep{lavalle1998rapidly, kuffner2000rrt} and Probabilistic RoadMap (PRM)~\citep{kavraki1996probabilistic} were introduced, and the planning was conducted in the configuration space~\citep{lozano1983spatial}. The sampling-based motion planning was simple to implement and was very quick to find solutions in open spaces, thus was adopted by many researchers and became one of the most often used planning methods. Since then, there have been many modifications to PRM and RRT algorithms, including methods trying to push towards optimal solutions~\citep{karaman2011sampling}, and methods to get around obstacles~\citep{amato1998obprm}, and other methods to bias the sampling to achieve better sampling efficiency~\citep{gamme2020batch}.

To further improve the quality of the paths, optimization processes were introduced to create better paths in the configuration space, such as CHOMP~\citep{ratliff2009chomp, zucker2013chomp} and STOMP~\citep{kalakrishnan2011stomp}. These optimizers build upon a known trajectory in configuration space, optimize towards better dynamic constraints, and can also help avoid obstacles. Many variations have been introduced to optimize towards different goals, including more human-like~\citep{dragan2015movement, marinho2016functional}, or more energy efficient~\citep{navarra2018energy, luo2015trajectory}.

The planning approach has also considered dynamic and control constraints in various forms. In addition to trajectory optimizers like CHOMP and STOMP, shortest paths have also been studied for different systems~\citep{dubins1957curves, reeds1990optimal, balkcom2002time, balkcom2006time,balkcom2018dubins, wang2018time,  wang2021towards,wang2022towards}. These shortest paths often are used as motion primitives, but the existing algorithms to find such paths can only find solutions when no obstacles exist.

\section{Workspace Initial Paths}

We can have a relatively accurate description of the constraints or obstacles in the workspace. From the simple polyhedrons to complex triangle meshes, we can usually derive helpful information from these descriptions, such as visibility or the existence of narrow passages. In contrast, even simple geometries can have complex boundaries in configuration space, especially when the robots are complex. Robots' structural information and workspace analysis cannot be reused in different planning scenarios.

When the workspace or the configuration space is relatively empty, sampling-based motion planning algorithms can quickly find feasible and sometimes even good quality configuration space trajectories. Currently, sampling-based motion planning algorithms often are applied in complex environments. Many modifications and additional steps are added to these simple algorithms to overcome various challenges in the complex configuration spaces.

\newpage

The workspace point-based paths serve as a {\em topological guidance} within the planning environment. Connectivity information is the prerequisite for the existence of paths. Many studies attempted to investigate the topological information in the configuration space (C-space). The topological properties in the C-space are challenging to extract and may not be closely related to workspace information. However, we can show that the Configuration space is also a fiber bundle, with the base space being the workspace. Therefore, the workspace topology can be mapped to the C-space topology using the fibers. The C-space topology can also be represented as a sub-space topology of the product topology of the workspace topology for different points on the robot. Knowing the workspace topology (connectivity) can be the basis for deriving C-space topology and furthering the robot's path geometry (trajectory).

The following would be the {\em key idea} of the proposed framework: segregating the topology planning from the geometric planning. We know that a valid trajectory must build on a connected topological path, which depends on a connected workspace path for points on the robot. From the topological path, i.e., the workspace point-based path, we can extrude the geometric paths and {\em coordinate} (constrained optimization) the paths for key points to satisfy the geometric relations on the robot. Eventually, a valid full (kinematic) trajectory is found for the robot. We show the basic outline of the framework and demonstrate that the framework has advantages over sampling-based motion planning in some scenarios.

A point or a sphere's path is easy to find in the workspace even when the obstacles are complex. Let us select $k$ key-points on the robot with appropriate radii $r_{\cdot}$ to approximate the geometric properties of the robot at these points. By {\em key-points}, we select the endpoints for each link, but other points can be included for other purposes. Distance constraints exist among these key points and are fixed with respect to the robot. Suppose the path for each key point has $r_{\cdot}$ clearance to the obstacles, then if we adjust the paths to satisfy the distance constraints among the key points at each time step, the resulting path is most likely to be valid for the entire robot. The main assumption is that following line-segment paths, when the two ends of a link have no collision, the entire link is most likely to be collision-free. We can then approximate the entire robot as a collection of links with different-sized endpoints.

\vspace{0.1in}
\begin{lemma}
In a 2D workspace, given a rigid link of length $l$ with end-points $p_a$ and $p_b$, and paths for $p_a$ and $p_b$ without collision and the distance between $p_a$ and $p_b$ remains $l$ along the path. Then, the area swept by the link is bounded by the paths followed by $p_a$ and $p_b$. If the paths for $p_a$ and $p_b$ are simple without loops, the link will not collide with the obstacle along the path.
\end{lemma}

\begin{proof}
Along the straight lines, the area swept by a link is bounded by the paths of the endpoints, i.e., the boundary of the swept area are the paths for the endpoints, plus the initial and final link position. Since the endpoint has no collision, the link has no collision.
\end{proof}


In 2D, it is easy to get the shortest point-based paths using visibility graph~\citep{ghosh1997recognizing}, and the path can easily be adjusted for circles with radii $r$ without collision. Moreover, this method makes it easy to find paths among narrow passages. Using segment paths for each key point, we then adjust the paths to satisfy the distance constraints so that the resulting paths can directly be applied to robots.

The same method has two main challenges in 3D. First, in 3D, it is not guaranteed that the visibility graph can produce the shortest paths. Second, the swept surface is no longer guaranteed to be collision-free even when the paths for the endpoints are collision-free. Compared to the 2D visibility graphs, the number of 3D shortest path candidates can grow exponentially. Direct search for the combination of the vertices and edges of the obstacles can be expensive. Methods like 3D visibility graphs can help to reduce the computation complexity but can still be expensive.

\newpage

Using triangulation, we introduce a workspace decomposition method that will categorize free regions of the workspace based on the number of adjacent obstacles. We only briefly describe the how one may get such a workspace decomposition and path, omitting some of the technical details and proof of correctness. More details are provided in~\citet{wang2022towards}. Each free-region is {\em approximately-convex}. If the collision-free paths are followed by the endpoints of a link belonging to the same free region, the path followed by the link is most likely to be collision-free. Some cases break the above assumption, but we will not focus on these cases in this work. We then extract the shortest paths between adjacent free regions to find the globally shortest path.

The categorization process can be described as follows. First, we construct triangulation among the free regions using the vertices of the obstacles. Triangulation (triangle or tetrahedron) can contain vertices from a single obstacle (and the boundary of the space), two, or three disjoint obstacles. If the triangulation contains vertices from no more than two obstacles, label them as adjacent to those obstacles. One free region can exist adjacent to a single or two obstacles, but concave regions on a single obstacle can have a triangulation component with all vertices from the same object. Therefore, the triangulation process starts after the convex hull of each obstacle is constructed, and intersecting convex hulls are merged after identifying common regions. If the triangulation component contains vertices from three or more disjoint obstacles, the corresponding free region may be adjacent to more than three obstacles. Therefore, we need to {\em merge} these free regions if they share common boundaries. We label the free region as a common region among these obstacles, whose vertices are part of the merged triangulation components.

Based on the categorization, we can assign a label to each free region, coded by the number of adjacent obstacles and the corresponding obstacle labels, using Godel numbering. We can also identify adjacent free regions with labels and construct an adjacency graph $G_{\mathcal{A}}$ among the free regions. If an obstacle has a sharp edge or pointy spikes, it is usually marked as the boundaries of different free regions unless a pillar-like geometry is fully contained in the area adjacent to only two obstacles. In such cases, we can treat such geometries as pillars and further partition the free region to avoid collision with the sweeping surface of the link endpoints. We will omit these cases in this work, as the work's main contribution is the workspace-based planning framework. The adjacency graph shows the connectivity relation of the workspace, and the paths on the graph correspond to workspace topological paths. These paths and the robot topology (structural information) form the C-space topology.

We then plan piece-wise linear paths for each key point within each free region. We will ensure that the start and the end of these piece-wise linear paths satisfy the distance constraints among the key points. Using these piece-wise linear paths, we then adjust the paths of the endpoints to satisfy the distance constraints.

There are flaws to the proposed approach. In particular, the distance traveled by the key points does not always correspond to the same amount of changes in configuration space. For example, consider a humanoid robot, a slight rotation in the hip joint may lead to the hands moving by a large amount. Though the change in configuration space is small, the distance traveled by the hand is significant. On the other hand, the hand's movement also symbolizes the amount of energy the hip joint needs to spend to move the hand, as the work done (distance traveled) in the workspace is one of the most straightforward ways of measuring the energy.

We can use the geometrical information to adjust the robot state to avoid collisions in the workspace. The same type of derivation is complex in configuration space. Again let us consider a humanoid robot, and we have the robot in collision with the obstacle at the right arm region. If we treat the robot as a collection of links, we can find the nearest key point to the collision and move the key point away from the collision point, especially if we have information such as penetration distance. Then, we can quickly adjust the robot state by moving the relevant part of the robot and leaving most of the robot in place.

\newpage

In contrast, in the configuration space, the inverse kinematic procedure must be performed to find the necessary movement for the robot. However, other collisions may be introduced if the movement is performed at an inappropriate joint. Building upon the workspace partition, we can find workspace paths as follows:
\begin{enumerate}
    \item Construct triangulation for the workspace and find free regions and their labels.
    \item Construct an adjacency graph for the free regions.
    \item Select a path on the adjacency graph connecting start and goal with a low estimated cost.
    \item Test if the sequence of free regions can allow a path for the robot; if not, return to step three.
    \item Find collision-free visibility path for the key-point with most neighbors on the robot.
    \item Find simple straight-line paths for other key points.
\end{enumerate}

In this paper, our primary focus is to show that the workspace optimization and coordination among multiple points are feasible, and the details for obtaining the piece-wise linear paths are only briefly described above. Many existing processes can  be employed to help find the paths, e.g., 3D visibility algorithms~\citep{durand19993d, durand20023d}.

When planning the piece-wise linear paths for the endpoints, we adapt a visibility principle: choose the straight-line connections with minimum crossings from the start to the targeted region among key points. In case such simple paths collides with the obstacles, we adapts the following strategy to find a valid path. First, we need to identify a hierarchical order for the key points, with one being the root or the primary key point. We refer such hierarchical order as the topology of the robot. Then, we will ensure the primary key point is collision free following a simple strategy, such as visibility. Then, we follow a Breadth First Search type of approach to rotate all other key points to stay clear of the obstacle, until all key points are collision free. This outcome of the proposed approach is similar to that of the obstacle-based PRM~\citep{amato1998obprm}, but the search space is the workspace. The desired outcomes of the two approaches are similar: find a valid configuration near the obstacles.

\begin{algorithm}[t]
      \caption{FindNextState ($P$)}\label{alg:FindNextState}
    \SetKwInOut{Input}{input}
    \SetKwInOut{Output}{output}
    \Input{$P$: the current robot state; $G$: visibility graph or free region adjacency graph}
    \Output{A new state}
    Pick base point $p_o$\;
    Find next point $q$ on the shortest point path for $p_o$\;
    $\hat{p}_o\leftarrow q$\;
    Adjust $\hat{p}_o$ to be collision free\;
    $Q\leftarrow p_o$ ($Q$: Queue)\;
    $V\leftarrow\emptyset$
    \While {Exist point not planned} {
        $p_c\leftarrow $ pop $Q$\;
        $V\leftarrow p_c$\;
        $N\leftarrow $ neighbors of $p_c$ on robot\;
        Push all element in $N$ into $Q$\;
        \For{$p\in N$} {
            $\vec{d}\leftarrow$ direction from $P(p)$ to $p_c$\;
            $l\leftarrow$ distance between $p$ and $p_c$\;
            $\hat{p}\leftarrow p_c + l\frac{\vec{d}}{\|\vec{d}\|}$\;
            Adjust $\hat{p}$ to avoid collision\;
        }
    }
    \Return $\hat{p_\cdot}$\;
\end{algorithm}

Once how the robot may travel from the start to the goal is found, i.e., a topological path (connectivity) is found, we can use Algorithm~\ref{alg:FindNextState} to find the intermediate states. The main goal is to admit collision-free piece-wise linear paths between adjacent states. Moreover, this algorithm attempts to avoid self-crossing through the visibility principle, reducing the need to add more intermediate states during planning. 

We assume the free region can allow the robot to pass in the above algorithm. The returned configuration can still have a collision with the obstacles. We can use penetration depth to find the direction for the adjustments. If a collision cannot be avoided, we will choose another topological path for planning.

\newpage

~
\vspace{0.3in}

\section{Workspace Path Optimization}

Once the paths for the endpoints are given, we need to adjust them to ensure they can be followed by the robot while satisfying the rigid link constraints, i.e., distance constraint. We first show the process for optimizing with respect to generic distance constraints. We then discuss other constraints and show some simple analyses and results. Note that the paths computed using the proposed algorithm are kinematic. The main advantage of the proposed approach is that the previously found paths provide good initial conditions for the optimization, so the process can be fast and is guaranteed to converge. Further optimizations can be introduced to improve the quality of the resulting paths.

\subsection{Geometric constraints}
\label{sec:constraint_geometry}

Distance constraint is one of the fundamental geometrical constraints on robots. When using configurations, the distance constraints are converted to a variable of joint angles and link lengths. One of the main motivations for using configurations is to reduce redundancy in the representation of the robot state. However, the configuration changes do not map geometric properties well in the workspace. For example, finding a good nearby configuration to avoid the current collision is difficult. In the workspace, however, finding a collision-free movement direction for each key point is much easier and interpolating the motion of other points to satisfy distance constraints while remaining collision-free.

In general, let us represent the distance constraint as follows, given two endpoints $p_a$ and $p_b$ of a link with length $l$:
\begin{eqnarray}
C: \|p_a - p_b\|_2 = l\label{eq:constraint_dis}\ .
\end{eqnarray}

Constraint Jacobian method is often used when planning for multiple points with different constraints. For example, in~\citet{littlefield2019kinodynamic}, the authors used constraint Jacobian $\frac{\partial C}{\partial t} = \mathbf{J}\mathbb{v}$ to compute the adjustments needed for the tensegrity robot endpoints to satisfy distance constraints and compute appropriate cable lengths. In most applications of constraint Jacobian, the $\mathbf{J}$ is used to help compute the velocity adjustment to the next time-step to satisfy constraint better. In the proposed approach, we will have a complete path for each point as the initial condition, and we compute the Jacobian for the constraints to find the adjustments needed for each point to satisfy the constraints and start the next iteration of optimization. So, instead of taking derivatives with respect to time, we take derivatives with respect to the coordinates of the key-points,
\begin{eqnarray}
\frac{\partial C}{\partial p^i_{\cdot}} = \mathbf{J}\Delta {p_{\cdot}}\label{eq:deltap}\\
\|p^i_a - p^i_b\| = d\label{eq:current_dis}\\
p^{i+1}_{\cdot} = p^{i}_{\cdot} + \frac{l-d}{2}\Delta {p_{\cdot}}\label{eq:updatep}
\end{eqnarray}

Given trajectories for $p_a$ and $p_b$, endpoints of a rigid link of length $l$. Using constraints in Eq.~\eqref{eq:constraint_dis}, and follow update rule in Eq.~\eqref{eq:deltap} and Eq.~\eqref{eq:updatep}, the resulting paths will satisfy the distance constraints. When no objectives are present, the update reduces violation. Since the updates between different points are independent, no conflict will occur. The constraints can be satisfied.

The satisfied distance constraint does not mean the overall trajectory is of good quality. Even though the initial conditions are the lower bounds for the path lengths in the proposed approach, the updates may lead to longer paths than necessary to satisfy the distance constraints. Therefore, in order to get good quality paths, we further introduce an objective function:
\begin{eqnarray}
\min\sum_{i=1}^{N-1}\|p_a(i+1)-p_a(i)\|\label{eq:minlengthObj}
\end{eqnarray}

We can further transform the above minimization objective into a different form:
\begin{eqnarray}
\min\|p_a(i+1) - p_a(i)\| + \|p_a(i+2)-p_a(i+1)\|
 -\|p_a(i+2)-p_a(i)\|\label{eq:minlengthObj_alt}
\end{eqnarray}

We can then introduce a small step-size $\alpha$ to reduce the total path length. When no obstacle exists, and no velocity limit exists for either endpoint, the shortest path can be a straight line, while the other follows a trajectory that satisfies the distance constraint. The above constraint is only for a single key point. We can add more key points to the objective function and achieve the smallest movement for all key points. Further, if there exists velocity constraint, we can  have:
\begin{eqnarray}
\dot{p_a(i)} = p_a(i+1) - p_a(i)\nonumber\\
C: \|\dot{p_a(i)}\| \leq v_a\label{eq:velLimit}
\end{eqnarray}

Eq.~\eqref{eq:velLimit} can be viewed as another distance constraint, and we can update the point positions using the same derivations as Eq.~\eqref{eq:deltap} and Eq.~\eqref{eq:updatep}.

Due to the distance constraints, the problem is not convex. Intuitively, one can discover that the gradient direction for satisfying the distance constraint and the minimization for the path lengths are almost perpendicular, satisfying the necessary condition for the convergence of the optimization process. Current experimental results show the quick convergence for the proposed process. In Figure~\ref{fig:x-r-path}, we show a $4$R and a $6R$ planar arm and the optimized shortest paths from straight-line initial paths using less than $10$ iterations.

\begin{figure}[h]

\vspace{-0.1in}

\begin{center}
\mbox{
    \includegraphics[width=2.6in]{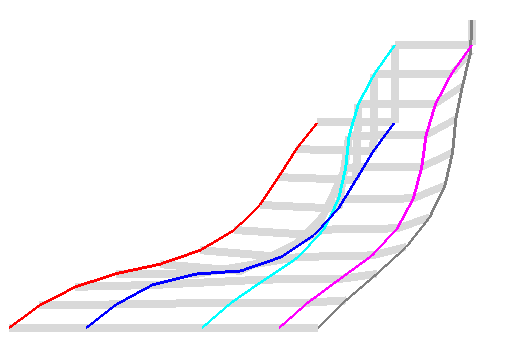}

    \includegraphics[width=2.6in]{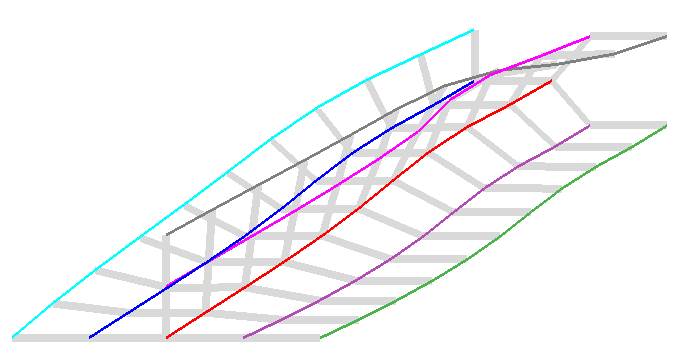}
}
\end{center}

\vspace{-0.25in}

    \caption{The paths returned by the proposed approach within less than $10$ iterations return smooth paths with small number of interpolated intermediate points. }
    \label{fig:x-r-path}
\end{figure}

Similar to CHOMP~\citep{ratliff2009chomp}, we can also integrate dynamic constraints into the optimization and the collision avoidance constraint. It is easy to integrate collision avoidance constraints, especially for endpoints. In 2D, collision avoidance constraint is not necessary within a link. In 3D, collision avoidance for endpoints does not guarantee the free of collision along with the link. So, we do not consider collision avoidance constraints in this work.

\begin{figure}[b!]

\vspace{-0.2in}

\begin{center}
\mbox{
    \subfigure[Linear interpolation may create self intersection.]{
    \includegraphics[width=2.5in]{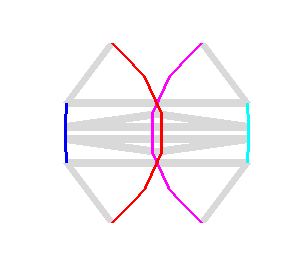}\label{fig:self-crossing}}\hspace{0.3in}
    \subfigure[Adding intermediate configurations avoids self-crossing. ]{
    \includegraphics[width=2.9in]{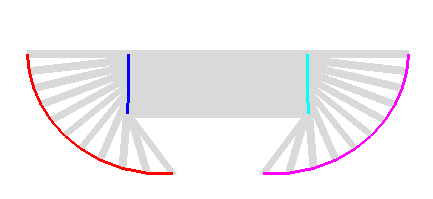}\label{fig:intermediate}}

}
\end{center}

\vspace{-0.25in}

    \caption{One may need to add intermediate states to avoid self-crossing. }
    \label{fig:self-intersection}\vspace{-0.2in}
\end{figure}

As with smoothness constraint, we can similarly add penalties to sudden large changes in acceleration, the second derivatives along the trajectory. However, the existing minimization and velocity bound constraint can produce smooth trajectories from the numerical experiments. For example, as shown in Figure~\ref{fig:x-r-path}, the trajectories for different endpoints (shown in different colors) are all smooth curves, even though all paths started as straight line segments.

Another constraint we need to consider is the no self-collision on the robot. Some intermediate configurations may contain self-intersection when we use the straight-line paths as initial solutions. As shown in Figure~\ref{fig:self-crossing}, the first and last links cross near the parallel configuration, and both the links would cross the second link along the path. In this case, we need to add intermediate states to avoid self-collision. For example, between the two folded states shown in Figure~\ref{fig:self-crossing}, we can easily find a path satisfying all constraints, as shown in Figure~\ref{fig:intermediate}. The self-intersection can be found through collision detection. Once the collision is identified, a new intermediate configuration will be added. We will follow the unfolding of the carpenter's rule procedure~\citep{abbott2009generalized} to find what is similar to an unfolded state for the colliding links. This problem, similar to singularity, is also present in other planning approaches.

In the experiments conducted, while minimizing the length of the paths, when only the distance constraint is present, the process always converges, and the Hessian matrix of the Lagrangian are always positive semi-definite. In this case, both the objective functions and the constraints are of the form $\|p_a-p_b\|_2 = \sqrt{(x_a-x_b)^2+(y_a-y_b)^2}$, therefore the Hessian contains elements of the form
\begin{eqnarray}
\frac{g(x_a, y_a, x_b, y_b)}{((x_a-x_b)^2+(y_a-y_b)^2)^{\frac{3}{2}}}
\end{eqnarray}
where $g(x_1, y_1, x_2, y_2)$ can be any of the following
\begin{eqnarray}
\pm (y_a-y_b)^2\\
\pm (x_a-x_b)^2\\
\pm (x_a-x_b)(y_a-y_b)
\end{eqnarray}
The diagonal elements are all positive square terms. If we consider only two points, where the objective is the minimization of the length for $p_a$ and $p_b$, and the constraint is the distance between $p_a$ and $p_b$, we can have the Hessian of the following form
\begin{eqnarray}
\Delta_{xx} L = \begin{bmatrix}
A | B \\
\hline
B | C
\end{bmatrix}
\end{eqnarray}
Both $A$ and $C$ have positive diagonal elements, and $B$ has negative diagonal elements. When more key points are being considered, the Hessian can be divided into more blocks. Each block either describes the second-order derivatives between two key points similar to $B$ or describes the second-order derivatives for the same key point at a different time-step similar to $A$. We want to prove the convergence of the optimization process in future work.

\subsection{Nonholonomic constraint}

Another common constraint type in motion planning is velocity constraints, such as no side-way velocity constraint for systems like Dubins and Reeds-Shepp vehicles. There do not exist good steering functions for these nonholonomic systems, and inverse kinematics is hard to find using a simple procedure. Methods like Rapid exploring Random Tree (RRT)~\citep{kuffner2000rrt} are commonly used to find paths for such systems. The quality of the paths is only guaranteed to approach optimal asymptotically, and the information is hard to reuse for planning between different start and goals. Existing solutions to optimal paths for such nonholonomic systems are only guaranteed to be resolution complete, and only numerical solutions exist except for some highly symmetric systems~\citep{dubins1957curves, reeds1990optimal, balkcom2002time, balkcom2006time, wang2012analytical, wang2012sampling}.


\newpage

In general, the nonholonomic constraints can be written as $W\cdot \dot{q} = c$, where $q$ is the state of the robot, and the $\dot{q}$ is the velocity of the robot. Constant $c$ is known to the system, and $W$ is a matrix that depends on the current state of the robot. For example, for Dubins car or Reeds-Shepp car, the constraint~can~be~written~as:
\begin{eqnarray}
\begin{bmatrix}
-\sin(\theta) \\ \cos(\theta) \\ 1
\end{bmatrix}\cdot\begin{pmatrix}
\dot{x} \\ \dot{y} \\ \dot{\theta}
\end{pmatrix} = 0
\end{eqnarray}
when we represent the robot using the $q = (x, y, \theta)$ configuration. In the workspace, we will need to convert the state representation from $q^c = (x, y, \theta)$ to $q^w = (x_1, y_1, x_2, y_2)$ for 2D rigid body, and change the corresponding parameter for $W$. Using $p_1 = (x_1, y_1)$ and $p_2 = (x_2, y_2)$, we can write the constraint to be the direction of $\dot{p_1}$ must be parallel to $\vec{p_1p_2}$ at next time-step.

When representing the velocity constraint, we need to normalize both vector to make sure the gradient is meaningful. Without loss of generality, let us assume $\|\dot{p_1}\|_2 = \|p_2-p_1\|_2 = 1$. We can have the following:
\begin{eqnarray}
\begin{bmatrix}
x_2 + \dot{x_2}\Delta t - x_1 - \dot{x_1}\Delta t\\
y_2 + \dot{y_2}\Delta t - y_1 - \dot{y_1}\Delta t
\end{bmatrix}\cdot \begin{pmatrix}
-\dot{y_1} \\ \dot{x_1}
\end{pmatrix} = 0\\
\Rightarrow\begin{bmatrix}
\frac{\partial C}{\partial\dot{x_1}} \\ \frac{\partial C}{\partial\dot{y_1}} \\ \frac{\partial C}{\partial\dot{x_2}} \\ \frac{\partial C}{\partial\dot{y_2}}
\end{bmatrix} = \begin{bmatrix}
y_2 - y_1 + \dot{y_2}\Delta t \\ x_1 - x_2 + \dot{x_2}\Delta t \\ -\dot{y_1}\Delta t \\ \dot{x_1}\Delta t
\end{bmatrix} = \mathbf{J}\Delta\dot{q}\label{eq:constraint_control}
\end{eqnarray}

\begin{figure}[h]

\vspace{-0.25in}

    \centering
    \subfigure{
    \includegraphics[width=3in]{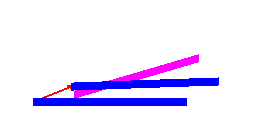}
    \label{fig:control_diff_rotation}
    }
    \subfigure{
    \includegraphics[width=3in]{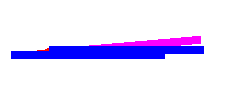}
    \label{fig:control_diff_trans}
    }

    \vspace{-0.3in}

    \caption{Updated velocities for different initial conditions. Only when the initial velocity is very close to the direction of the link vector (right), the gradient shown in Eq.~\eqref{eq:constraint_control} will push the robot to follow a rotational control (left). }
    \label{fig:arm_opt}\vspace{0.2in}
\end{figure}

Only when the current $\dot{p_1}$ is translation along the direction of $\vec{x_1x_2}$, the gradient of $\frac{\partial C}{\partial\dot{y_1}}$ will be along the direction of $\vec{x_1x_2}$. Otherwise, the gradient will push the robot more and more toward a rotational control, as shown in Figure~\ref{fig:arm_opt}. Following the derivation, we can see that the translation is not stable for Dubins-like vehicles. Even though the straight line segments are good initial conditions for path lengths, it is not good for control optimization, which can easily be shown in the velocity space (Figure~\ref{fig:vel_plot}). Directly using the proposed process will push the controls towards rotation and eventually near a turn-drive-turn-like trajectory, oscillating between optimization iterations.

\newpage

\begin{figure}[t]

\begin{center}
\mbox{
    \includegraphics[width=3in]{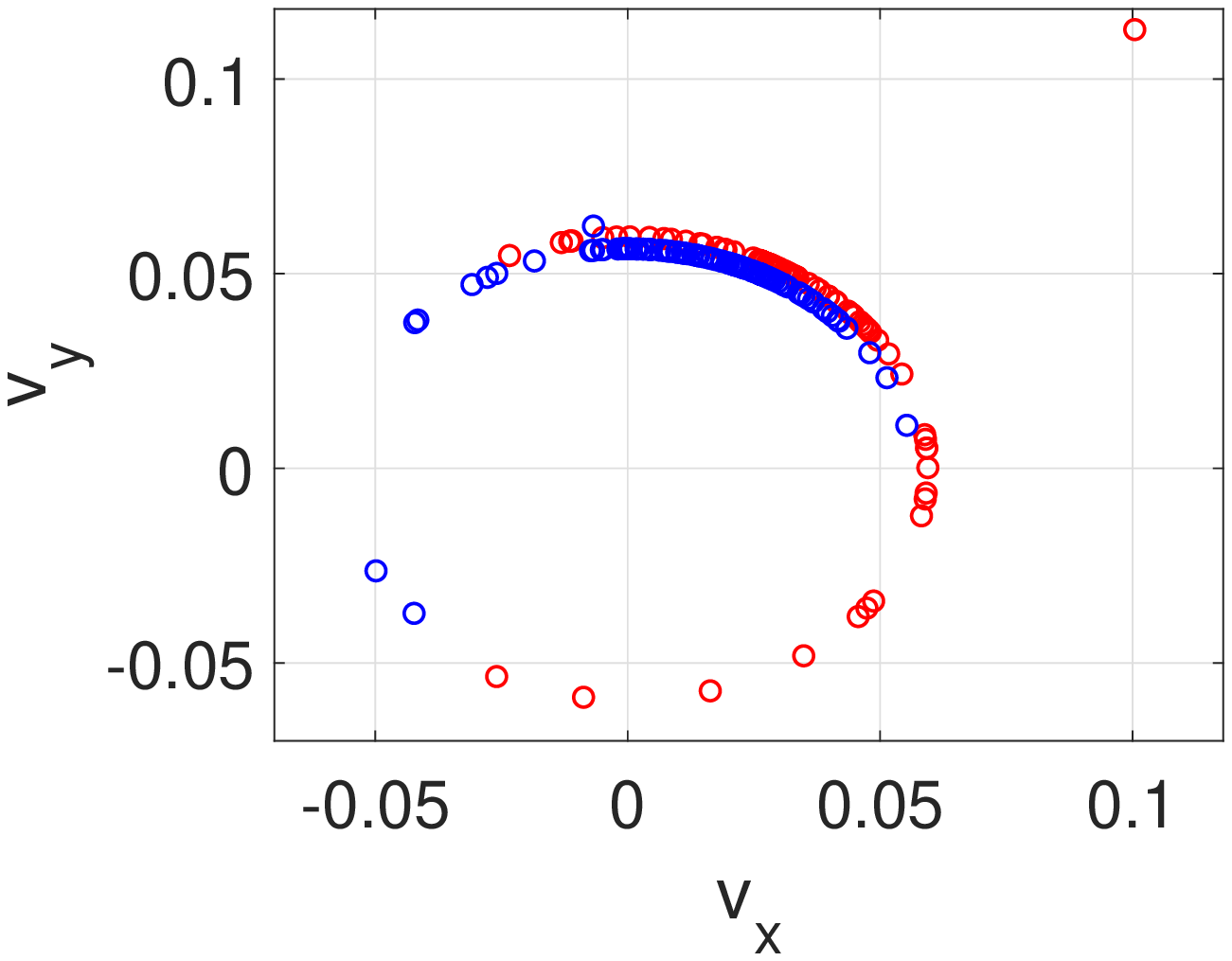}

    \includegraphics[width=3in]{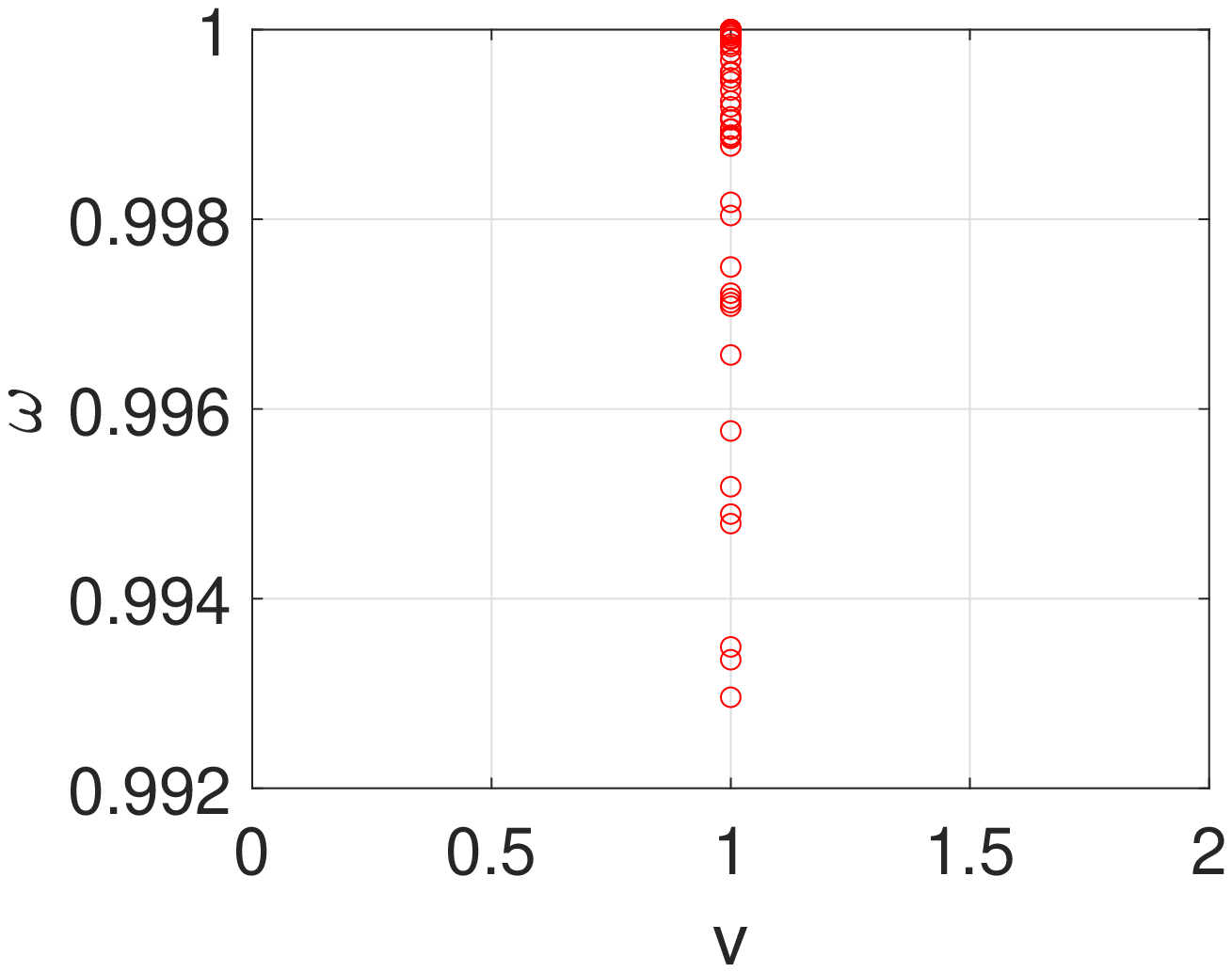}
}
\end{center}

\vspace{-0.2in}

    \caption{The velocity for two key-points after satisfying distance constraints (left), and the same velocities converted to $(v, \omega)$ space~(right). }
    \label{fig:vel_plot}\vspace{0.2in}
\end{figure}

Converting to the $(v, \omega)$ velocity space, we will see the following velocity scatter shown on the right of Figure~\ref{fig:vel_plot}. The optimal control theory proved that the velocity would maximize towards the corners in the control space to minimize the trajectory length. In this case, the corners are $(1, 1)$ and $(1, -1)$ after scaling. The straight-line path provides a good initial condition for the geometric optimization, but the resulting trajectory is bad for optimizing against the nonholonomic constraints.

The nonholonomic constraint of $W\cdot\dot{q} = c$ is known to be challenging to integrate. In the analysis above, we can see that the derivative of this constraint is also unstable, especially near translation for the Dubins-like system. The main challenge to integrating the nonholonomic constraint is whether the resulting trajectory can reach the goal. Therefore, It is very likely that the velocity update needs to consider the constraints coming from reaching the goal. Future work's main focus is finding a stable optimization process to find paths that satisfy the nonholonomic constraints.

\section{Experiments}

We show some experiments for the complete proposed process, from partitioning the workspace to finding piece-wise linear paths for multiple key points, and then optimize to approach optimal paths for the given robot from start to goal. We further compare the proposed approach with sampling-based motion planning algorithms conducted in the configuration space, such as PRM and RRT. The proposed approach aims to find collision-free good, quality kinematic paths for the given robot in the environment from scratch, rather than optimizing paths after an initial valid path is found. In this section, we consider only the distance constraints on the robot.

Let us consider the 2D environments with multiple random polygonal obstacles and some with narrow passages. We consider planning for a four-link arm with a moving base. All links have the same width. In Figure~\ref{fig:env1}, we show an example of the planning. The environment has only a few obstacles, but the clearance between obstacles is small. The goal is at the bottom right corner.

\begin{figure}[t]
    \centering
    \subfigure{
    \includegraphics[width=2.7in]{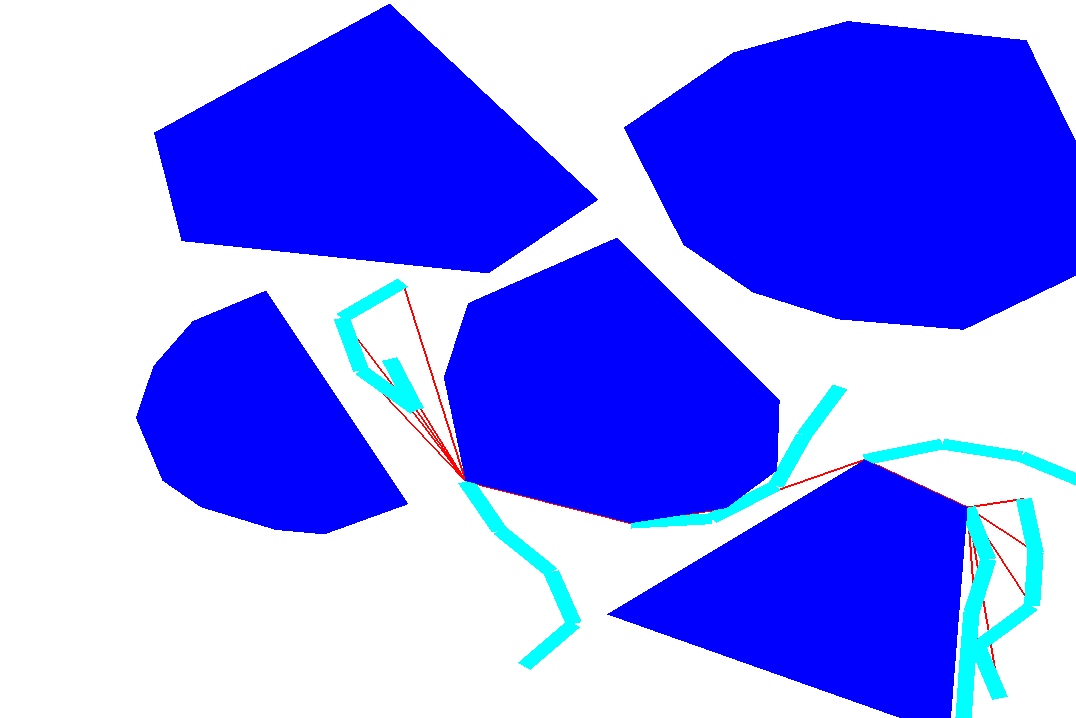}
    \label{fig:env1-steps}
    }\hspace{0.2in}
    \subfigure{
    \includegraphics[width=3.2in]{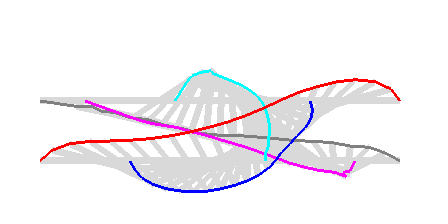}
    \label{fig:four-4-reverse}
    }
    \caption{A 2D environment with a 4R arm, starting from the configuration among the obstacles to the goal on the lower right corner. The path and intermediate configuration found using the proposed method are shown on the top. The 4R arm needs to reverse to reach the goal configuration near the goal. An illustration of the optimized reversing trajectory is shown at the bottom. }
    \label{fig:env1}\vspace{0.2in}
\end{figure}

Using the proposed approach, we first find the visibility path in the environment and then use Algorithm~\ref{alg:FindNextState} to find the intermediate states. We then use the optimization process to interpolate the path between each adjacent configuration. The direct line paths have a collision with the obstacle between the second and the third intermediate state, so another intermediate state is introduced to avoid the collision. The method is to find the visibility path between two endpoints and add the intermediate collision-free state as a transition. The proposed approach has several intermediate states compared to sampling-based motion planning, but the number of intermediate states is much smaller than the number of samples used in a PRM or RRT.

Once the feasible collision-free intermediate states are found, we use the process in Section~\ref{sec:constraint_geometry} to find a feasible path between adjacent states. In Figure~\ref{fig:env1}, using Algorithm~\ref{alg:FindNextState}, we find the last intermediate state with the links almost parallel to the goal state, but the entire orientation for the serial arm is inverted. Therefore, a rotation in place is needed to adjust the orientation. At the bottom of Figure~\ref{fig:env1}, we show a trajectory for the key points to achieve a similar change in orientation.

We will need to overcome the relatively long and narrow passage using sampling-based motion planning. Since the free region is small, the collision-free samples are hard to find in the given environment, further improving the cost of sampling-based planning. In the environment shown in Figure~\ref{fig:env1}, no valid path is found after placing $4000$ samples. The runtime for placing these samples takes many minutes, while the method proposed in this work took a fraction of a second to find a valid path.

Of course, in relatively open space, the sampling-based motion planner, such as PRM or PRM$^{*}$, can find a solution relatively fast, while the quality of the path takes a while to approach the quality found by the proposed approach. When a path is found using PRM$^{*}$, the path quality is very close to that found by our proposed method. However, in narrow passage environments, as shown in Figure~\ref{fig:env1}, the time cost of PRM$^{*}$ is much greater than the proposed method for the serial arms.

The proposed approach works well when the obstacles are polygons or polyhedrons with few vertices. The proposed approach may become slower when the environments have many obstacles made with a massive number of triangle meshes. Instead of directly checking for geometric collisions, the proposed approach needs to analyze the vertices to find piece-wise linear paths. Therefore, in a more complex environment, we may use a convex hull of the obstacles to find how to connect the start and goal and use the geometry of each obstacle to check for collision later.

In the meanwhile, the proposed method can also work well when part of the robot needs to remain in place while exploring the feasible state of other parts, such as humanoid walking. Though not explored experimentally in this paper, we will try to integrate the nonholonomic and other physical constraints, such as gravity and inertia, into the model.

\section{Conclusion}
\label{sec:conclusion}

We propose a planning framework that mainly plans in the workspace. Rather than considering the geometry of the entire robot, we use the geometries of the obstacles to partition the workspace and generate good piece-wise linear paths for various {\em key-points} on the robot. Then, optimizations are conducted to coordinate the paths to yield valid paths for the robot. The proposed process plans for points in the workspace and then applies simple optimization steps to converge quickly on good-quality paths. We compared the proposed approach with sampling-based algorithms and found that the proposed approach can be faster than the sampling-based motion planning to generate similar quality paths and have more advantages in environments with narrow passages. The proposed approach can be perceived as a divide and conquer approach in the workspace, where the intermediate states subdivide the overall problem into locally solvable geometric problems.

\vspace{0.5in}

\bibliographystyle{plainnat}
\bibliography{refs}

\end{document}